\documentclass{article}

\usepackage{arxiv}

\usepackage[utf8]{inputenc} 
\usepackage[T1]{fontenc}    
\usepackage{hyperref}       
\usepackage{url}            
\usepackage{booktabs}       
\usepackage{amsfonts}       
\usepackage{nicefrac}       
\usepackage{microtype}      
\usepackage{lipsum}		
\usepackage{bbm} 
\usepackage{graphicx}
\usepackage{natbib}
\usepackage{doi}
\usepackage{multirow}
\usepackage{bbm}

\usepackage{microtype}
\usepackage{graphicx}
\usepackage{subfigure}
\usepackage{booktabs} 

\usepackage{amsmath}
\usepackage{amssymb}
\usepackage{mathtools}
\usepackage{amsthm}

\usepackage[capitalize,noabbrev]{cleveref}

\usepackage{wrapfig}
\theoremstyle{plain}
\newtheorem{theorem}{Theorem}[section]
\newtheorem{proposition}[theorem]{Proposition}

\newtheorem{fact}[theorem]{Fact}
\theoremstyle{definition}

\theoremstyle{remark}

\usepackage[textsize=tiny]{todonotes}

\newcommand{\out}[1]{}
\usepackage{algorithm}
\usepackage{algorithmic}

\usepackage{amsmath,amsfonts,bm}

\newcommand{\nin}{n_\text{in}}
\newcommand{\nout}{n_\text{out}}
\newcommand{\nlr}{n_\text{lr}}

\newcommand{\W}{\mathbf{W}}

\newcommand{\B}{\mathbf{B}}
\newcommand{\x}{\mathbf{x}}

\def\eqref#1{equation~\ref{#1}}

\def\1{\bm{1}}

\DeclareMathAlphabet{\mathsfit}{\encodingdefault}{\sfdefault}{m}{sl}
\SetMathAlphabet{\mathsfit}{bold}{\encodingdefault}{\sfdefault}{bx}{n}

\usepackage{booktabs}
\usepackage{amssymb,amsmath}

\title{Functional-level Uncertainty Quantification for Calibrated Fine-tuning on LLMs}

\author{ {\hspace{0.1mm}Ruijia Niu}\textsuperscript{ 1 2} \\
	\And{\hspace{0.1mm}Dongxia Wu}\textsuperscript{ 1 2}\\
    \And{\hspace{0.1mm}Rose Yu}\textsuperscript{ 1 2}\\
    \And{\hspace{0.1mm}Yi-An Ma}\textsuperscript{ 2 1}\\
}

\date{}

\newcommand{\blfootnote}[1]{\begingroup%
\renewcommand\thefootnote{}\footnotetext{#1}%
\addtocounter{footnote}{-1}%
\endgroup}

\renewcommand{\headeright}{}

\renewcommand{\shorttitle}{Functional-level Uncertainty Quantification for Calibrated Fine-tuning on LLMs}

\hypersetup{
pdftitle={Functional-level Uncertainty Quantification for Calibrated Fine-tuning on LLMs},
pdfauthor={Ruijia Niu, Dongxia Wu, Rose Yu, Yi-An Ma},
pdfkeywords={Uncertainty Quantification, Large Language Models, Mixture of Experts, Parameter Efficient Fine Tuning},
}

\begin{document}
\maketitle

\setcounter{footnote}{1}
\footnotetext{Department of Computer Science and Engineering, University of California San Diego, La Jolla, California, USA}
\setcounter{footnote}{2}
\footnotetext{Halıcıoğlu Data Science Institute, University of California San Diego, La Jolla, California, USA}
\blfootnote{Preprint. Correspondence to: Yian Ma <yianma@ucsd.edu>.}

\begin{abstract}
Accurate uncertainty quantification in large language models (LLMs) is essential for reliable confidence estimation, yet fine-tuned LLMs often become overconfident under limited adaptation data. Existing uncertainty methods for PEFT-based LLMs are largely post hoc, estimating uncertainty after fine-tuning rather than improving how adapters specialize to task-specific input-output relationships. We propose Functional-Level Uncertainty Quantification for Calibrated Fine-Tuning (UQ4CT), which calibrates uncertainty over the functional space induced by prompt-dependent mixtures of LoRA experts. UQ4CT implements this perspective through a mixture-of-experts fine-tuning framework, where a calibration loss aligns functional-level confidence with predictive correctness during training. Across four multiple-choice benchmarks and two open-ended generative QA tasks, UQ4CT reduces Expected Calibration Error (ECE) by over $25\%$ while preserving high accuracy. Under distribution shift, UQ4CT maintains superior calibration and competitive accuracy, demonstrating improved reliability and generalization for fine-tuned LLMs.
\end{abstract}

\section{Introduction}
Quantifying the credibility of outputs has been one of the most important problems around large language models (LLMs)\citep{chang2024survey}. In particular, fine-tuned LLMs often struggle with overconfidence in their outputs due to limited training data, failing to reflect the true credibility of their answers\citep{xiao2022uncertainty,he2023preserving,tian2023just,openai2023gpt4}. Such overconfidence can assert misinformation with high certainty, making it difficult for users to discern truth from falsehood. This has become a crucial problem in safety-critical decision making and scientific domains where data is relatively limited, such as formal proof generation, climate science, and healthcare~\citep{singhalLargeLanguageModels2022,wuBloombergGPTLargeLanguage2023,lampinenPassiveLearningActive2023,li2022pre}. Methods that enhance uncertainty quantification for fine-tuned LLMs are therefore essential to ensure trustworthy predictions.

A key challenge in LLM uncertainty quantification is balancing accuracy, calibration, and efficiency: uncertainty estimates should be calibrated without degrading predictive performance or slowing generation. Existing methods often rely on prompt perturbations to measure prediction variance~\citep{hou2023decomposing,gao2024spuq} or multiple sampled completions to estimate disagreement~\citep{farquhar2024detecting}. However, these approaches largely assume the model is already aligned with the target data distribution, limiting their ability to capture fine-tuning-induced uncertainty and generalization gaps under limited adaptation data. Their reliance on multiple forward passes also introduces substantial computational overhead, reducing scalability.
Beyond prompt-level approaches, Bayesian methods and ensemble-based uncertainty quantification have been established for fine-tuned LLMs, often in conjunction with low-rank adaptation(LoRA) ~\citep{hu2021lora}. Methods such as Monte Carlo dropout \citep{gal2016dropout}, checkpoint ensembles \citep{chen2017checkpoint}, deep ensembles \citep{lakshminarayanan2017simple,wang2023lora,zhai2023uncertainty}, and Laplace-LoRA \citep{yang2024bayesian} apply Bayesian inference or ensembling over LoRA parameters to capture uncertainty arising from model adaptation and limited data.

While Bayesian and ensemble methods estimate uncertainty after fine-tuning by analyzing the learned parameter space, they do not address the limitations caused by sparse data during fine-tuning. This post hoc perspective can miss uncertainty that arises when adapting to new tasks with limited data. 

\begin{wrapfigure}{r}{0.5\textwidth}
  \centering
  \includegraphics[width=0.5\textwidth]{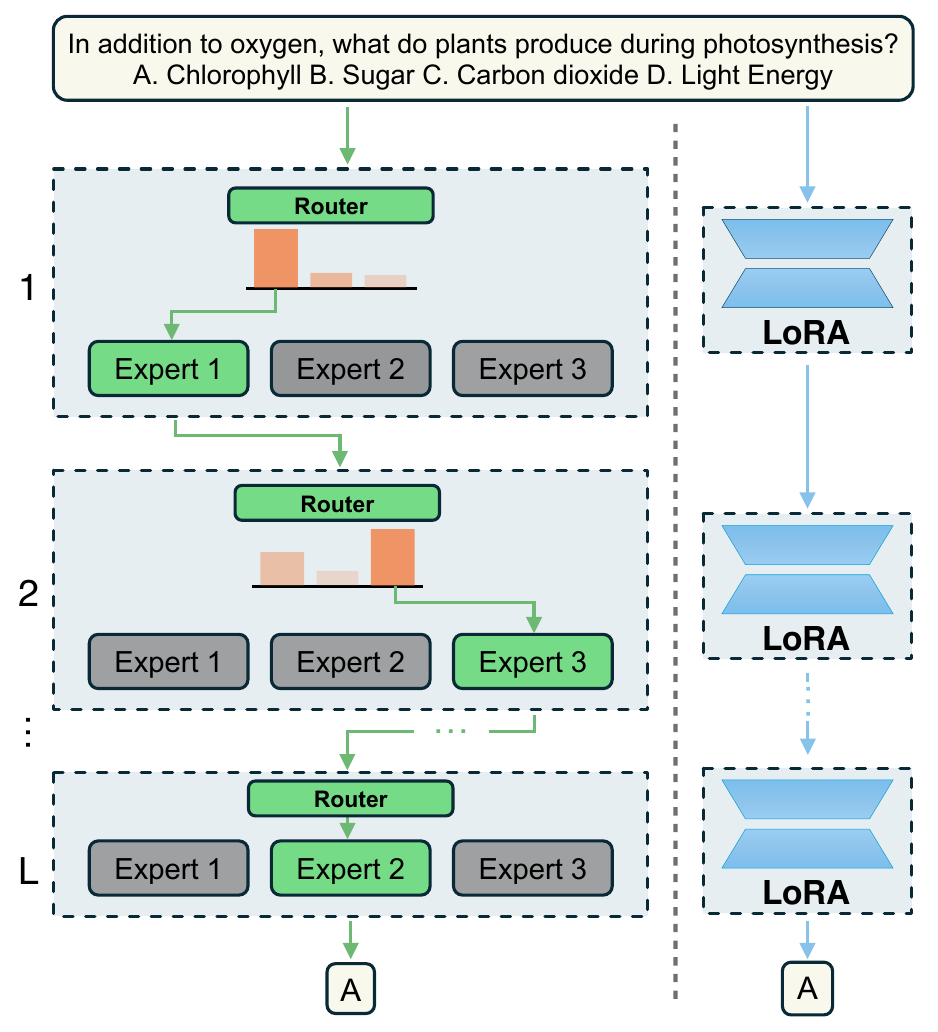}
  \caption{Left: The MoE architecture captures diverse functional relationships by dynamically routing inputs to different expert modules based on the prompt. Right: The standard LoRA approach lacks such functional-space diversity, limiting its capacity to capture variations in the functional relationships during fine-tuning.}
  \label{fig:MixLoRA}
\end{wrapfigure}

To overcome this, we shift focus from parameter-space to functional-space uncertainty quantification. The functional space encompasses the input-output mappings the model can realize, capturing the true variability in its predictions. By calibrating uncertainty at this level during fine-tuning, we ensure the model’s confidence better reflects its actual predictive reliability.

We therefore introduce Functional-Level Uncertainty Quantification for Calibrated Fine-Tuning (UQ4CT), a method that explicitly calibrates functional-level uncertainty in LLMs during fine-tuning. As demonstrated in \Cref{fig:MixLoRA}, UQ4CT leverages ensembles of LoRA modules at each layer to construct a rich set of basis functions. We then employ a Mixture-of-Experts (MoE) architecture~\citep{li2024mixlora} to hierarchically combine these basis functions, forming a flexible functional space. During fine-tuning, UQ4CT jointly learns the LoRA expert parameters and calibrates the prompt-dependent function mixture to align functional-level uncertainty with predictive correctness, enabling the model to output calibrated distributions over the functional space.

During inference, LoRA experts offer diverse functional relationships acquired during fine-tuning, while MoE routers dynamically select the most relevant experts for each input prompt. This selection is guided by functional-level uncertainty calibration performed throughout fine-tuning, which aims to optimize the choice of the correct functional relationship for each prompt. More accurate expert selection enables the model to learn diverse functional relationships. As a result, the model’s uncertainty estimates become better aligned, enhancing calibration without compromising accuracy.
To summarize, our contributions include:
\begin{itemize}
    \item A novel uncertainty quantification approach for LLMs with MoE architecture during fine-tuning to quantify functional-level uncertainty and align with the probability of predictive correctness, which mitigates overconfidence issue and improves generalizability.
    \item A new calibration loss incorporating predictive correctness to dynamically align the prompt-dependent LoRA mixture for better uncertainty estimation.
    \item Hierarchical decomposition of functional-level uncertainty into layer-wise mixture weights with guarantee that our calibration loss aligns mixture weights with predictive correctness.
    \item Empirically, UQ4CT achieves over $25\%$ ECE reduction on $3$ common-sense reasoning tasks and $1$ domain-specific question answering task while preserving high accuracy. Under distribution shift, UQ4CT maintains superior calibration on $2$ common-sense reasoning tasks and $4$ domain-specific tasks. Experiments on open-ended generative QA further suggest that the functional-space calibration perspective extends beyond classification-style settings.
\end{itemize}
\section{Preliminaries}
\paragraph{Low-rank Adaptation (LoRA).}
LLMs have numerous large weight matrices to perform matrix multiplication, denoted as $\W_0 \in \mathbb{R}^{\nout\times\nin}$ that maps inputs $\mathbf{x}$ to outputs $\mathbf{h}$. \cite{hu2021lora} proposes LoRA, which fixes $\W_0$ and introduces a low-rank perturbation $\Delta \W$ to the weight matrix:
\begin{align}  \label{eq:lora}
  \mathbf{h} &= \W_0 \x + \Delta \W \x = \W_0 \x + \B \mathbf{A} \x.
\end{align}

\vspace{-1mm}

Here, $\Delta \W$ is calculated as the product of two matrices, $\B \in \mathbb{R}^{\nout\times\nlr}$ and $\mathbf{A}\in\mathbb{R}^{\nlr\times \nin}$ where $\nlr$ is significantly smaller than $\nin$ or $\nout$. For example, we use $\nlr = 32$ while $\nin = \nout = 4096$ for the Llama3.1-8B model \citep{grattafiori2024llama}. Therefore, the total number of LoRA parameters for this $\Delta \W$ is $\nlr (\nin + \nout)$, which is far smaller than the parameter count of the full matrix, $\nin\nout$. One of the key motivations of incorporating LoRA to fine-tune LLMs is the vast amount of memory cost reduction compared with fine-tuning on the full model. For an LLM with $7$ billion parameters, maintaining the average gradient
and average squared gradients for optimization multiplies the memory required by a factor of $3$ compared to simply loading model weights. LoRA greatly mitigates this memory cost as the tripled memory consumption only applies to LoRA adapters.

\paragraph{Mixture of Experts (MoE).}

LoRA Mixture-of-Experts \citep{li2024mixlora, wu2024mixture} is an efficient approach to scale the number of parameters while maintaining the same computational bounds. LoRA MoE utilizes the top-k router to assign each token to the LoRA experts \citep{lepikhingshard}. The router is a linear layer that maps the input hidden state $\mathbf{h}$ to a probability distribution of candidate experts. 

The plain transformer block in a large language model consists of the $q,k,v$ encoding layers ($\text{FFN}_{q,k,v}$),
layer norm (LN) and the feedforward layer (FFN), together with residual connections. Formally, given $\mathbf{h}^1$ as the tokenized input text, the output of $\ell$-th layer is generated as: 

\begin{equation}
    \mathbf{z}^\ell = f_{attn}(\text{FFN}_{q,k,v}(\text{LN}(\mathbf{h}^{\ell-1}))) + \mathbf{h}^{\ell-1}, \quad
    \mathbf{h}^\ell = \text{FFN}(\text{LN}(\mathbf{z}^\ell)) + \mathbf{z}^\ell.
\end{equation}

Here, $f_{attn}$ represents the attention calculation step.

Let $\mathbf{h}^\ell \in \mathbb{R}^{1\times d}$ $(1 \leq \ell \leq L)$ denote the output hidden state at the $\ell$-th layer of the LLM, where $L$ is the number of LLM layers and $d$ is the hidden dimension. With $\mathbf{W}_r^\ell$ as the trainable router weight at layer $\ell$, the top-k gate router $\Tilde{R}(\cdot)$ chooses $k$ experts with highest probability given a hidden state $\mathbf{h}^\ell$:

\begin{equation} \label{eq:moe_router}
    \Tilde{R}^\ell(\mathbf{h}^\ell) = \text{Keep-Top-k}(\text{Softmax}(\mathbf{W}^\ell_r \cdot \mathbf{h}^\ell)).
\end{equation}

Finally, we obtain the final MixLoRA prediction with:
\begin{equation} \label{eq:moe_out}
    \text{MixLoRA}(\mathbf{h}^\ell) = \sum_{k=1}^K \Tilde{R}^\ell(\mathbf{h}^\ell)_k E^\ell_k(\mathbf{h}^\ell), \quad
    E^\ell_k(\mathbf{h}^\ell) = \mathbf{W}_{pre}^\ell \cdot \mathbf{h}^\ell + \mathbf{B}^\ell_k \mathbf{A}^\ell_k \cdot \mathbf{h}^\ell.
\end{equation}
where $\mathbf{W}^\ell_{pre}$ is the frozen pretrained weight at layer $\ell$ and $\mathbf{B}^\ell_k \mathbf{A}^\ell_k$ is the k-th LoRA expert.

With MixLoRA defined in Equation \ref{eq:moe_out}, we can apply MixLoRA layers at $q,k,v$ encoding and FFN layers:
\begin{equation}
    \mathbf{z}^\ell = f_{attn}(\text{MixLoRA}_{q,k,v}(\text{LN}(\mathbf{h}^{\ell-1}))) + \mathbf{h}^{\ell-1}, \quad
    \mathbf{h}^\ell = \text{MixLoRA}(\text{LN}(\mathbf{z}^\ell)) + \mathbf{z}^\ell.
\end{equation}

\section{Methodology}
\label{sec:method}
The high-level goal of UQ4CT is to leverage the ensemble of prompt-dependent LoRA mixture-of-experts (MoE) to guide and calibrate the confidence of the model during fine-tuning. By quantifying the variability in how different LoRA experts are combined for each input, UQ4CT enables the model to adaptively select expert mixtures that reflect the true uncertainty in its predictions. Our approach not only encourages the model to exploit confident expert combinations for accurate predictions but also promotes exploration of alternative experts when uncertainty is high, ultimately leading to better-calibrated and more reliable model outputs.

We note that the MixLoRA MoE structure used in UQ4CT is not part of the LLM architecture itself; rather, it is a lightweight, modular fine-tuning adapter that operates entirely at the LoRA adapter level. UQ4CT only requires LoRA-compatible linear layers, which are standard in modern decoder-based LLMs. With the base model remains unchanged, it is as easy to apply as standard LoRA.
\begin{figure*}[t!]
  \centering
  \includegraphics[width=1\linewidth,trim={0, 0, 0, 0}]{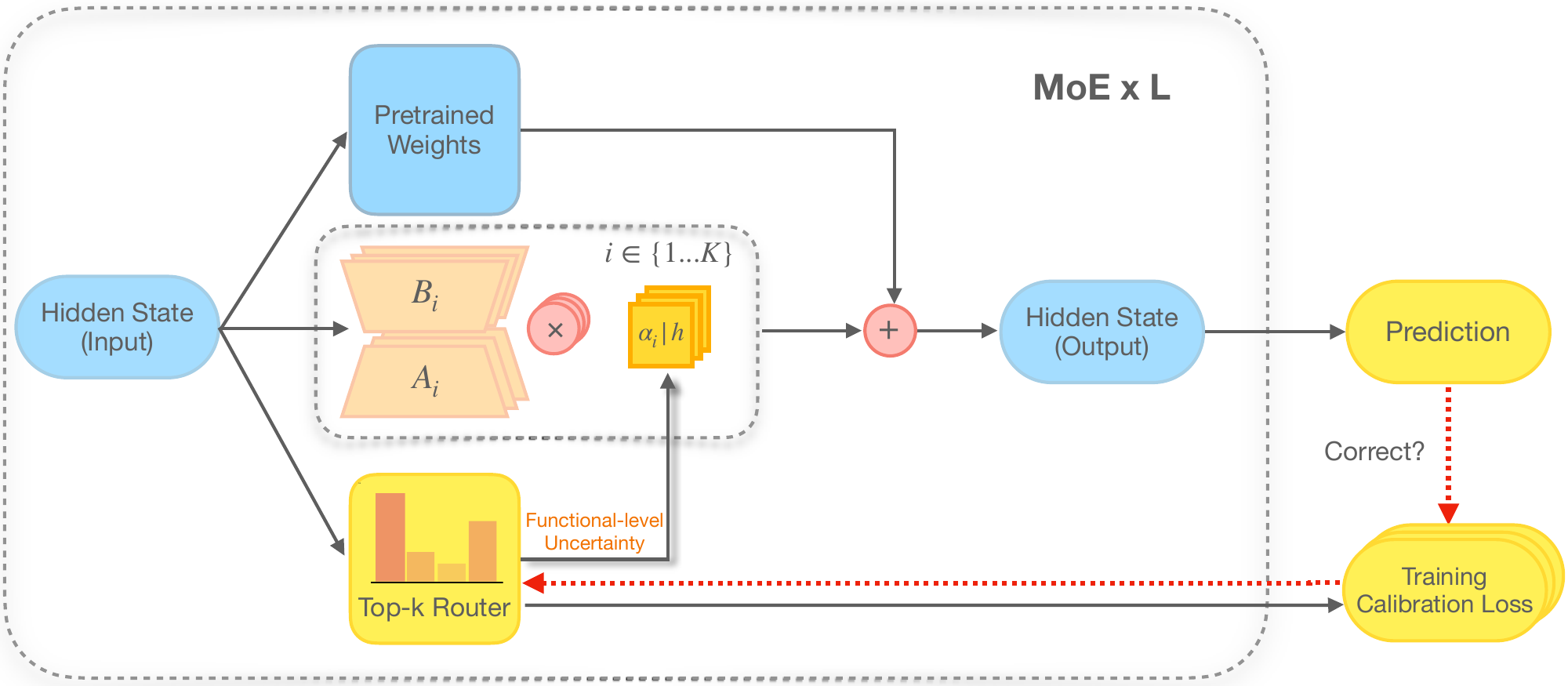}
  \caption{MoE architecture to capture functional-level uncertainty. LoRA experts ($B_i$, $A_i$) capture diverse functional bases, while the top-$k$ router assigns mixture weights based on the input hidden state. The calibration loss aligns functional-level uncertainty with prediction correctness, encouraging confident expert selection for correct predictions and higher uncertainty for incorrect ones.}
    \label{fig:Struct}
\end{figure*}

\subsection{Decomposition of the Functional Space}
\label{sec:decomp}
Given the immense size of both the pre-training dataset and the model, we posit that the pretrained network contains submodules capable of expressing a wide range of functional relationships present in the data. During fine-tuning, our focus is on the functional space spanned by the model, which can be effectively captured as a mixture of LoRA experts, each representing a distinct basis function. However, naively decomposing the full functional space by considering all possible combinations of these experts quickly leads to a combinatorial explosion.

\textbf{Naive Decomposition.}
Denote the input prompt as $\mathbf{x}$ and the functional map from input to output as $f$. A straightforward approach is to express the function as a sum over all possible compositions of $K$ submodules per layer, across $L$ layers:
\begin{equation}
f(\mathbf{x}) = \sum_{k^1,\cdots, k^L=1}^{K} \alpha_{k^1,\dots,k^L}g_{k^L}^L\bigg( \cdots \bigg( g_{k^{l}}^{l} \big(\dots g_{k^1}^1(\mathbf{x}) \big) \bigg) \bigg),
\label{eq:naive_decomp}
\end{equation}
where each $g_{k^\ell}^\ell$ denotes a particular variant (e.g., LoRA-adapted) of the $\ell$-th block and
\begin{align}
g_{k^\ell}^\ell (\mathbf{h}^\ell) = \mathbf{h}^\ell + E_{k^\ell}^\ell \big( f_{trans}^\ell (\mathbf{h}^\ell) \big),
\end{align}
here $f_{trans}^\ell$ represents the necessary non-fine-tuned operations within a transformer block (i.e. layer norm, attention calculation, etc.) and as defined in Eq.~\ref{eq:moe_out}, $E_{k^\ell}^\ell$ denotes the parameterized adaptation associated with the $k^\ell$-th expert in the $\ell$-th layer.

However, this naive decomposition is intractable in practice, as it requires keeping track of $K^L$ mixture weights $\alpha_{k^1,\ldots,k^L}$. This is an exponential growth in the number of parameters with respect to both the number of layers $L$ and the number of submodules per layer $K$, which makes the direct approach computationally infeasible for realistic network sizes.

\textbf{Hierarchical Decomposition.}
To address the combinatorial explosion of mixture weights, we instead propose a hierarchical decomposition. Here, the mixture at each layer is formed independently, and the output of each layer is a weighted sum over its submodules, with the weights themselves being layer-specific:
\begin{equation}
f(\mathbf{x}) = \sum_{k^L=1}^{K} \alpha^L_{k^L} g_{k^L}^L \bigg( \cdots \Bigg( \sum_{k^{l}=1}^{K} \alpha^{l}_{k^{l}} g_{k^{l}}^{l}  \bigg( \dots \sum_{k^1=1}^{K} \alpha^1_{k^1} g_{k^1}^1(\mathbf{x}) \bigg) \Bigg)\Bigg).
\label{eq:hierarch_decomp}
\end{equation}
Instead of needing $K^L$ mixture weights, this hierarchical structure only requires $K \cdot L$ weights $\alpha^\ell_{k^\ell}$, one for each submodule in each layer. This dramatically reduces the parameters required and makes the decomposition tractable, while still enabling a rich set of compositional functions.

\textbf{Dynamic, Input-Dependent Routing.}
To further enhance expressivity and efficiency, we allow the mixture weights to depend dynamically on the input at each layer. Specifically, we set the mixture weights to be a sparse routing function $R^\ell$ of the hidden state $\mathbf{h}^\ell$ at each layer, where
$\alpha_{k^\ell}^{\ell} = \alpha_{k^\ell}^{\ell}(\mathbf{h}^\ell) = R_{k^\ell}^\ell(f_{trans}^\ell (\mathbf{h}^\ell)).$

Substituting this definition into the hierarchical mixture, the overall function $f(\mathbf{x})$ becomes:
\begin{equation} \label{eq:sub_alpha}
f(\mathbf{x}) = \sum_{k^L=1}^{K} R_{k^L}^L(f_{trans}^L (\mathbf{h}^L)) g_{k^L}^L\Bigg( \dots \Bigg(\sum_{k^1=1}^{K} R_{k^1}^1(f_{trans}^1 (\mathbf{h}^1)) g_{k^1}^1(\mathbf{x}) \Bigg) \Bigg).
\end{equation}

This recursive formulation shows that at each layer, the submodules are weighted according to the input-dependent routing function, which adapts based on the current hidden state.

For a layer-wise perspective, the computation at each layer $\ell$ can be explicitly written as:
\vspace{-1mm}
\begin{align}
\mathbf{h}^{\ell+1} = \sum_{k^{\ell}=1}^{K} 
R_{k^\ell}^\ell(f_{trans}^\ell (\mathbf{h}^\ell)) g^{\ell}_{k^\ell}(\mathbf{h}^\ell).
\label{eq:hierachy_decomp}
\end{align}
\vspace{-1mm}
Here, $R^\ell$ selects the most relevant submodules for a given input, allowing the network to adaptively compose its computation path at each layer in a sparse fashion.

\subsection{Quantifying FLC with MixLoRA}
In the previous section, we have established a parsimonious representation of the functional space.
In this section, we choose a simple function to encode the functional level uncertainty and provide the intuition as follows. We first note that given a fixed MoE model architecture, the larger weight a mixture component has, the more certain we are about that component. 

The model uncertainty of the MixLoRA architecture is quantified by considering perturbations $\Delta f(\mathbf{x})$ to the model $f(\mathbf{x})$. Following the discussion and notation in Sec.~\ref{sec:decomp}, we can show  that these perturbations are instantiated in the space of the mixture weights $\alpha$ (as defined in Eq.~\ref{eq:hierarch_decomp} and \ref{eq:sub_alpha}):

\begin{fact}[Model Perturbation Structure]
\label{fact:model}
Under regularity assumptions on the residual connection architecture, perturbations $\Delta f(x)$ to the model $f(x)$ approximately decompose as:
\[
\Delta f(x) \approx \sum_{\ell=1}^{L} \sum_{k^\ell=1}^K \Delta \alpha_{k^\ell}^\ell(\mathbf{h}^\ell) \cdot g_{k^\ell}^\ell(\mathbf{h}^\ell).
\]
\end{fact}


In high-dimensional settings, the basis functions are approximately orthogonal. Thus, the perturbation $\Delta f(x)$ is entirely represented by the set $\left\{ \Delta \alpha_{k^\ell}^\ell(\mathbf{h}^\ell) \right\}_{k^\ell=1,\dots,K}^{\ell =1,\dots,L}$. 
Therefore, the functional-level confidence (FLC) can be generally modeled as a linear function over the mixture weights:
\[
\text{FLC}(x) = U\left( \left\{\alpha_{k^\ell}^\ell(\mathbf{h}^\ell) \right\}_{k^\ell=1,\dots,K}^{\ell =1,\dots,L}\right),
\]
where $U(\cdot)$ denotes a linear aggregation function. Details of derivation is presented in Appendix \ref{app:derivation}.

In practice, as illustrated in Figure~\ref{fig:Struct}, the top-$k$ router at each layer produces a sparse routing distribution that dynamically mixes the basis functions represented by the LoRA experts conditioned on the current hidden state. The retained top-$K$ routing weights directly reflect the concentration of expert selection and therefore provide a functional-level confidence signal. Following the MoE routing mechanism in Eq.~(\ref{eq:moe_router}) and Eq.~(\ref{eq:moe_out}), we use a top-$2$ router. For each layer $\ell$, we compute the router probabilities and retain the two largest entries:
\begin{equation}
    \tilde{R}^\ell(\mathbf{h}^\ell)=\mathrm{KeepTop2}\!\left(\mathrm{Softmax}(\mathbf{W}_r^\ell \mathbf{h}^\ell)\right),
    \quad
    \mathrm{FLC}(x)=\frac{1}{L}\sum_{\ell=1}^{L}\sum_{i=1}^{2}\tilde{R}_i^\ell(\mathbf{h}^\ell).
\end{equation}

Given an input prompt $x$, we aggregate the retained routing weights over the selected experts and across all $L$ layers to estimate the functional-level confidence (FLC) for the predicted next token, where $i$ indexes the two experts selected by the top-$2$ router at layer $\ell$.

\subsection{Calibration Loss}

The FLC model provides a principled way to calibrate the mixture parameters against predictive accuracy, enabling better alignment between output distributions and true model confidence. Specifically, for MoE top-$k$ routers, we design the following calibration loss for training:
\begin{equation} \label{eq:calibration}
    \mathcal{L}_{\mathrm{cal}} = \left(\mathbbm{1}{\{\text{MixLoRA}(x) = y^*\}} - \text{FLC}(x) \right)^2.
\end{equation}
The first term is an indicator function that equals $1$ if the model prediction matches the ground truth $y^*$ for prompt $x$, and $0$ otherwise, corresponding to a one-hot definition of ground truth confidence.

This calibration loss encourages functional-level confidence (FLC) to reflect predictive correctness. As shown in Figure~\ref{fig:Struct}, correct predictions push FLC toward $1$, while incorrect predictions push toward $0$. Optimized over the data distribution, this objective aligns FLC with the probability of prediction correctness.
We formally state this property as follows and present the proof in Appendix \ref{app:derivation}: 

\begin{proposition}[Truthfulness of Calibration Loss]\label{prop:calib}
Let the calibration risk be defined as the expectation of the calibration loss over the data distribution:
\begin{equation*}
\bar{\mathcal{L}}_{\mathrm{cal}} = \mathbb{E}_{(x,y^*)\sim\mathcal{D}}\left( \mathbbm{1}\{\mathrm{MixLoRA}(x) = y^*\} - \mathrm{FLC}(x) \right)^2.
\end{equation*}
If this calibration risk is optimized over the data distribution $\mathcal{D}$, then the optimal solution is $\mathrm{FLC}(x) = \mathbb{P}(\mathrm{MixLoRA}(x) = y^*(x))$; that is, the optimally trained FLC corresponds to the probability that the model's prediction is correct.
\end{proposition}
\normalsize
\subsection{Fine-tuning with the Total Loss}

Our proposed calibration loss $\mathcal{L}_{\mathrm{cal}}$ improves predictive reliability by adaptively balancing expert exploitation and exploration according to functional-level uncertainty. As shown in Figure~\ref{fig:Struct}, $\mathcal{L}_{\mathrm{cal}}$ aligns uncertainty with predictive correctness, increasing the router probability of the selected expert for correct predictions (exploitation) and decreasing it for incorrect ones (exploration).

Ideally, when the $K$ LoRA experts collectively capture the relevant functional relationships in the data during fine-tuning via cross-entropy loss, $\mathcal{L}_{\mathrm{cal}}$ further guides the model to select appropriate mixtures of LoRA experts conditioned on the input $\mathbf{x}$. This targeted selection enables the model to match the data distribution more closely and provides more calibrated uncertainty estimates.

To ensure balanced expert utilization, we incorporate a load balancing loss $\mathcal{L}_b$ as proposed by \citet{li2024mixlora}. Our overall loss function is:
\begin{equation} \label{eq:total_loss}
\mathcal{L} = \mathbf{CE} + \gamma \cdot \mathcal{L}_b + \beta \cdot \mathcal{L}_{\mathrm{cal}},
\end{equation}
where $\mathbf{CE}$ is the cross-entropy loss, and $\gamma$, $\beta$ are hyperparameters for the auxiliary terms. We fix $\gamma$, $\beta$ to $1$ for our experiments. Details of $\mathcal{L}_b$ are provided in Appendix~\ref{app:load}.

\section{Related Work}
\begin{table*}[t!]
\small
\centering
\caption{Performance comparison of different methods fine-tuned with Llama3.1-8B across three common sense reasoning tasks and a domain-specific task. UQ4CT shows substantial ECE improvements while maintaining high accuracy.}
\label{table:results}
\vspace{0.1cm}
\begin{tabular}{l l c c c c}
\toprule
\textbf{Metrics} & \textbf{Methods} & \textbf{ARC-E} & \textbf{ARC-C} & \textbf{OBQA} & \textbf{ClimateQA} \\
\midrule
\multirow{9}{*}{\textbf{ACC} $\uparrow$}
& Base Model & $87.27$ & $74.32$ & $72.80$ & $68.64$ \\
& LoRA      & $\underline{88.82_{1.82}}$ & $78.21_{0.73}$ & $88.00_{1.22}$ & $78.25_{1.29}$ \\
& MC Drop   & $88.16_{1.75}$ & $77.14_{0.69}$ & $87.12_{1.18}$ & $78.19_{1.35}$ \\
& Ensemble   & $\mathbf{89.14_{1.31}}$ & $78.81_{0.96}$ & $86.47_{0.42}$ & $78.53_{2.98}$ \\
& MixLoRA   & $87.74_{0.36}$ & $78.56_{1.87}$ & $\underline{88.27_{0.50}}$ & $\underline{79.94_{1.29}}$ \\
& LA        & $86.22_{3.52}$ & $78.00_{3.76}$ & $86.00_{6.01}$ & $79.82_{3.48}$ \\
& BLoB(Mean)      & $88.71_{0.82}$ & $79.37_{0.71}$ & $87.60_{1.04}$ & $79.02_{0.50}$ \\
& BLoB(N=10)      & $87.96_{0.62}$ & $\mathbf{80.08_{1.55}}$ & $87.13_{0.88}$ & $79.02_{0.50}$ \\
& UQ4CT     & $88.66_{0.20}$ & $\underline{79.60_{1.21}}$ & $\mathbf{88.40_{0.35}}$ & $\mathbf{79.97_{0.85}}$ \\
\midrule
\multirow{9}{*}{\textbf{ECE} $\downarrow$}
& Base Model & $13.76$ & $11.30$ & $11.39$ & $14.58$ \\
& LoRA      & $6.55_{1.70}$  & $14.07_{0.68}$ & $7.30_{0.43}$  & $13.70_{1.50}$ \\
& MC Drop   & $6.48_{1.74}$  & $14.12_{0.71}$ & $7.24_{0.39}$  & $13.11_{1.46}$ \\
& Ensemble   & $7.08_{0.73}$  & $13.71_{1.29}$ & $8.63_{0.38}$  & $14.69_{0.84}$ \\
& MixLoRA   & $7.79_{0.45}$  & $13.71_{1.90}$ & $6.58_{0.21}$  & $14.68_{0.09}$ \\
& LA        & $7.63_{1.71}$  & $8.92_{4.16}$ & $11.97_{5.97}$ & $11.48_{1.66}$ \\
& BLoB(Mean)    & $4.89_{0.32}$ & $11.26_{1.13}$  & $6.83_{0.90}$  & $12.74_{0.88}$ \\
& BLoB(N=10)      & $\mathbf{3.35_{0.50}}$ & $\underline{6.81_{1.43}}$  & $\underline{3.84_{1.00}}$  & $\underline{11.96_{2.57}}$ \\
& UQ4CT     & $\underline{3.97_{0.78}}$ & $\mathbf{4.43_{0.82}}$ & $\mathbf{3.34_{1.60}}$ & $\mathbf{9.36_{2.77}}$\\
\bottomrule
\end{tabular}
\end{table*}

\paragraph{Mixture of LoRA Experts.}
LLMs have achieved impressive performance across diverse NLP tasks~\citep{brown2020language, Hoffmann2022TrainingCL, Touvron2023LLaMAOA, Touvron2023Llama2O}, with instruction fine-tuning~\citep{Chung2022ScalingIL, Iyer2022OPTIMLSL, zheng2023judging} further boosting their adaptability for conversational AI~\citep{wu2023brief, gpt4}. However, scaling LLMs increases the resource demands of full fine-tuning. Parameter-efficient fine-tuning (PEFT) methods~\citep{peft} such as LoRA~\citep{Hu2021LoRALA} and its variants~\citep{kopiczko2024vera, hyeonwoo2023fedpara, renduchintala2023tiedlora, zhang2023adalora, liu2024dora} reduce adaptation costs by updating a subset of parameters.

Recent advances combine PEFT with the Mixture-of-Experts (MoE) framework~\citep{Jacobs1991AdaptiveMO, wang2020deep}, which sparsely activates expert subnetworks for greater model capacity and specialization. MoE-based LLMs leverage expert routing and parameter-efficient adaptations to target new domains or tasks efficiently. Notably, methods such as MoRAL~\citep{yang2024moral}, LoRAMoE~\citep{dou2024loramoe}, PESC~\citep{wu2024parameterefficient}, MoE-LoRA~\citep{luo2024moelora}, and MixLoRA~\citep{li2024mixlora} optimize domain-specific routing, mitigate forgetting, and enable scalable, high-throughput training and inference with mixtures of LoRA experts.

\paragraph{Uncertainty Quantification in LLMs.} 
Established uncertainty quantification methods have been studied in conjunction with the LoRA structure for LLMs. Monte-Carlo dropout \citep{gal2016dropout} interprets dropout in neural networks as approximate Bayesian inference in deep Gaussian processes, allowing uncertainty estimates to be obtained from existing LoRA adapters without modifying them. Checkpoint ensemble \citep{chen2017checkpoint} utilizes predictions from multiple LoRA checkpoints saved during a single fine-tuning process to calibrate uncertainty. Deep ensemble \citep{lakshminarayanan2017simple,wang2023lora,zhai2023uncertainty} combines the predictions from multiple LoRA adapters for better uncertainty calibration. Laplace-LoRA \citep{yang2024bayesian} applies Bayesian inference via Laplace approximation to the LoRA parameters after fine-tuning, resulting in improved calibration and uncertainty estimates. Bayesian Low-Rank Adaptation by Backpropagation (BLoB) \citep{wang2024blob} extends the LA method by jointly optimizing the mean and covariance of LoRA parameters via backpropagation throughout fine-tuning.

Prompt-perturbation and resampling-based approaches have also been explored for UQ in LLMs. These methods estimate uncertainty by measuring prediction variability under different prompt formulations or sampled input variants, without altering model parameters \citep{farquhar2024detecting, hou2023decomposing, gao2024spuq}. This line of work leverages the inherent sensitivity of LLMs to input perturbations as a means to assess model confidence, providing a complementary perspective to parameter-based methods.

\vspace{-4mm}
\section{Experiments}
\paragraph{Datasets.}
We evaluate on four multiple-choice QA benchmarks: OpenBookQA (OBQA)~\citep{OpenBookQA2018}, ARC-Easy (ARC-E), ARC-Challenge (ARC-C)~\citep{clark2018think}, and ClimateQA, a domain-specific climate science benchmark. To assess generalization to open-ended generation, we further evaluate on TruthfulQA~\citep{lin2021truthfulqa} and TriviaQA~\citep{joshi2017triviaqa}. For robustness under distribution shift, we group domain-specific MMLU subtasks~\citep{hendrycks2020measuring} into four professional-domain benchmarks: Computer Science (CS), Engineering (Eng), Law, and Health, with details provided in Appendix~\ref{app:OOD}. Models are fine-tuned on the public training split and evaluated on the corresponding test split for each benchmark.

\begin{table*}[t]
\small
\centering
\caption{Performance comparison of different methods fine-tuned on the OBQA dataset with Llama3.1-8B across two smaller distribution shift (DS) tasks and four larger distribution shift tasks. UQ4CT shows substantial ECE improvements while maintaining high accuracy.}
\label{table:DS_test}
\vspace{0.1cm}
\setlength{\tabcolsep}{3.5pt} 
\begin{tabular}{l l c c c c c c c}
\toprule
 & & \textbf{ID} & \multicolumn{2}{c}{\textbf{Smaller DS}} & \multicolumn{4}{c}{\textbf{Larger DS}} \\
\cmidrule(lr){3-3} \cmidrule(lr){4-5} \cmidrule(lr){6-9}
\textbf{Metrics} & \textbf{Methods} & \textbf{OBQA} & \textbf{ARC-C} & \textbf{ARC-E} & \textbf{CS} & \textbf{Eng} & \textbf{Law} & \textbf{Health} \\
\midrule
\multirow{8}{*}{\textbf{ACC} $\uparrow$} 
& LoRA       & $88.0_{1.22}$ & $77.8_{0.16}$ & $86.7_{0.77}$ & $\mathbf{55.8_{0.52}}$ & $54.3_{3.30}$ & $44.9_{0.23}$ & $58.8_{0.21}$ \\
& MC Drop    & $87.1_{1.18}$ & $77.1_{2.01}$ & $86.9_{2.42}$ & $54.4_{1.58}$ & $54.1_{1.82}$ & $45.0_{0.76}$ & $58.3_{1.46}$ \\
& Ensemble   & $86.5_{0.42}$ & $78.2_{0.90}$ & $85.4_{0.47}$ & $53.8_{1.02}$ & $52.4_{0.56}$ & $45.0_{0.20}$ & $60.6_{0.57}$ \\
& MixLoRA    & $\underline{88.3_{0.50}}$ & $78.1_{0.45}$ & $86.7_{0.35}$ & $53.1_{1.14}$ & $\underline{54.7_{2.28}}$ & $45.0_{1.46}$ & $\underline{60.9_{1.04}}$ \\
& LA         & $86.0_{6.01}$ & $78.7_{0.55}$ & $86.4_{0.76}$ & $\underline{54.7_{1.82}}$ & $53.6_{2.77}$ & $44.9_{1.03}$ & $59.7_{0.94}$ \\
& BLoB(Mean)  & $87.6_{1.04}$ & $\underline{79.5_{1.10}}$ & $86.6_{0.65}$ & $51.2_{0.99}$ & $48.6_{1.44}$ & $39.9_{7.85}$ & $57.0_{3.51}$ \\
& BLoB(N=10) & $87.1_{0.88}$ & $\mathbf{79.8_{1.06}}$ & $\underline{87.2_{0.79}}$ & $52.8_{1.28}$ & $51.9_{3.13}$ & $43.8_{4.60}$ & $58.5_{5.33}$ \\
& UQ4CT      & $\mathbf{88.4_{0.35}}$ & $79.0_{0.56}$ & $\mathbf{87.8_{0.47}}$ & $53.3_{0.61}$ & $\mathbf{61.1_{3.20}}$ & $\mathbf{45.4_{0.50}}$ & $\mathbf{61.1_{1.48}}$ \\
\midrule
\multirow{8}{*}{\textbf{ECE} $\downarrow$} 
& LoRA       & $7.30_{0.43}$ & $14.8_{0.62}$ & $9.6_{0.69}$ & $21.0_{3.05}$ & $24.1_{3.63}$ & $29.3_{1.98}$ & $24.0_{1.81}$ \\
& MC Drop    & $7.24_{0.39}$ & $13.4_{2.15}$ & $10.2_{1.89}$ & $20.8_{3.26}$ & $24.1_{0.77}$ & $29.1_{0.64}$ & $21.6_{3.92}$ \\
& Ensemble   & $8.63_{0.38}$ & $15.4_{0.46}$ & $10.7_{0.56}$ & $14.0_{3.18}$ & $17.4_{1.98}$ & $19.9_{2.95}$ & $16.1_{2.07}$ \\
& MixLoRA    & $6.58_{0.21}$ & $14.5_{0.55}$ & $9.9_{0.20}$ & $17.1_{3.06}$ & $17.8_{2.80}$ & $21.6_{4.07}$ & $18.0_{2.78}$ \\
& LA         & $11.97_{5.97}$ & $7.2_{0.5}$ & $6.4_{0.42}$ & $13.7_{2.14}$ & $15.5_{2.0}$ & $\underline{19.0_{0.54}}$ & $15.7_{2.30}$  \\
& BLoB(Mean)  & $6.83_{0.90}$ & $11.37_{1.94}$ & $6.6_{1.65}$ & $17.2_{2.72}$ & $18.5_{2.82}$ & $22.6_{1.26}$ & $16.9_{3.02}$ \\
& BLoB(N=10) & $\underline{3.84_{1.00}}$ & $\underline{5.8_{0.96}}$ & $\mathbf{3.0_{0.87}}$ & $\underline{11.5_{2.76}}$ & $\underline{14.9_{1.93}}$ & $19.7_{3.21}$ & $\underline{14.5_{3.38}}$ \\
& UQ4CT      & $\mathbf{3.34_{1.60}}$ & $\mathbf{3.6_{1.44}}$ & $\underline{3.6_{1.32}}$ & $\mathbf{10.8_{3.73}}$ & $\mathbf{13.2_{1.86}}$ & $\mathbf{18.1_{4.40}}$ & $\mathbf{13.2_{4.06}}$ \\
\bottomrule
\end{tabular}
\end{table*}

\paragraph{Experiment Setup.}

We implement UQ4CT with PyTorch \citep{paszke2019pytorch}, extending the MixLoRA repository in \citep{li2024mixlora}. We use the Llama-3.1-8B \citep{grattafiori2024llama} as our base model. In particular, we apply MixLoRA to query, key, value and output layers, together with the feed-forward networks in LLaMA-3.1-8B (gate layer, down layer and up layer). Details are provided in Appendix \ref{app:implement_detail}. 

\paragraph{Baselines.} We compare UQ4CT with state-of-the-art uncertainty estimation methods along with naive fine-tuning applied to the LoRA adapters of LLMs, including \textbf{LoRA} \citep{hu2021lora}, \textbf{Monte Carlo (MC) Dropout} \citep{gal2016dropout}, \textbf{Deep Ensemble} \citep{lakshminarayanan2017simple}, \textbf{Laplace-LoRA (LA)} \citep{yang2024bayesian}, \textbf{Bayesian Low-Rank Adaptation by Backpropagation (BLoB)} \citep{wang2024blob} and \textbf{MixLoRA} \citep{li2024mixlora}. Note that for \textbf{BLoB(N=10)}, the method performs $10$ forward passes with differently sampled LoRA parameters for each question, which is an unfair computational budget advantage compared against UQ4CT with only $1$ forward pass.

\vspace{-2mm}

\paragraph{Evaluation.}
We evaluate prediction accuracy on the validation set across all four tasks. For uncertainty calibration, we use the expected calibration error (ECE; \cite{guo2017calibration}; more details in \ref{app:ece}) to measure the alignment between predicted probabilities and actual outcomes.

To assess robustness under distribution shifts, we fine-tune models on the OBQA dataset and evaluate them following \citet{yang2024bayesian}. We use ARC-C and ARC-E to represent smaller distribution shifts, as these datasets focus on general science reasoning similar to OBQA but are more challenging and diverse. For larger shifts, we utilize the four aforementioned domain-specific MMLU subtasks, which span a wide range of expertise from elementary to professional levels. This domain represents a greater distribution shift from the general common sense focus of OBQA.

The in-distribution setting evaluates alignment with the target task, while the distribution-shift setting assesses generalization to novel domains. Together, they provide a practical evaluation of both task-specific performance and robustness to out-of-distribution inputs.

\subsection{In-distribution Performance}
As shown in Table \ref{table:results}, UQ4CT achieves notable gains in uncertainty calibration across diverse tasks while maintaining competitive accuracy relative to baseline approaches. The accuracy remains on par with or above the best baselines, demonstrating that improved uncertainty quantification does not come at the expense of predictive performance.

The primary gains are in calibration. UQ4CT consistently achieves the lowest ECE across benchmarks, with especially large improvements on challenging tasks where baselines remain poorly calibrated. Notably, UQ4CT requires only a single forward pass, whereas BLoB ($N=10$) uses ten forward passes yet still yields higher ECE on most benchmarks.

To further validate our method, we include results from fine-tuning both LLaMA-3.1-8B (main text) and Mistral-7B (Appendix \ref{app:mistral}). Across both models, UQ4CT delivers substantial and consistent improvements in uncertainty calibration. A key advantage of UQ4CT is that it incorporates uncertainty calibration directly during fine-tuning, incurring minimal computational overhead compared to other UQ methods that require costly repetitive sampling or post-hoc adjustments.

\subsection{Performance Under Distribution Shift}
Due to the sparse nature of the fine-tuning data, real world deployment of LLMs often requires the model to be robust to out-of-distribution knowledge \citep{ouyang2022training,touvron2023llama,touvron2023llama2}. Therefore, we evaluate the performance of UQ4CT along with other baseline models fine-tuned on the OBQA dataset under smaller and larger distribution shift scenarios. 

Table~\ref{table:DS_test} presents the distribution shift evaluations. UQ4CT achieves substantial ECE improvements while maintaining high accuracy across smaller and larger distribution shifts. For smaller shifts, UQ4CT's ECE remains comparable to the in-distribution scenario. Under larger shifts, UQ4CT attains the lowest ECE among baselines and delivers competitive accuracy on all domain-specific tasks. These results demonstrate that aligning uncertainty at the functional level with predictive correctness improves generalizability and mitigates overconfidence in fine-tuned models.


\subsection{Ablation Studies}
We conduct ablation studies to assess the contribution of the calibration loss, $\mathcal{L}_{cal}$. We analyze sensitivity to $\beta$ in Equation~\ref{eq:total_loss} (Appendix~\ref{app:beta_sensitivity}), showing that small values already yield substantial ECE gains and that $\beta=1$ provides the best trade-off. We further examine the number of active LoRA experts (Appendix~\ref{app:num_experts}), incremental calibration weighting (Appendix~\ref{app:incremental_weighting}), and comparisons with prompt-perturbation methods (Appendix~\ref{app:ppc}). Finally, we provide a detailed peak memory analysis comparing LoRA, UQ4CT, and BLoB in Appendix~\ref{app:memory_analysis}, showing that UQ4CT and BLoB have nearly identical peak memory (${\sim}$21.35~GB), both only ${\sim}$7.5\% above standard LoRA.

\textbf{Open-Ended Generative Extension.} \label{sec:generative}
To explore the potential of UQ4CT extending beyond multiple-choice tasks, we evaluate on TruthfulQA and TriviaQA using an F1-based correctness proxy for calibration. For controlled comparison, we first fine-tune the base model with MixLoRA and then apply all baselines to the same fine-tuned backbone. Setup details, including the adapted calibration loss and baselines, are provided in Appendix~\ref{app:generative_setup}.
\begin{table}[t]
\small
\centering
\caption{All UQ methods applied to the \textbf{same MixLoRA-fine-tuned} Llama-3.1-8B-Instruct on open-ended generative benchmarks. Sampling-based methods (marked with *) use $N=5$ samples. SE: Semantic Entropy, SD: Semantic Density, VC: Verbalized Consistency.}
\label{table:generative}
\vspace{0.1cm}
\resizebox{0.9\linewidth}{!}{
\begin{tabular}{l l c c c c c c}
\toprule
\textbf{Task} & \textbf{Metrics} & \textbf{SE*} & \textbf{SD*} & \textbf{SeqProb} & \textbf{Verb.} & \textbf{VC*} & \textbf{UQ4CT} \\
\midrule
\multirow{2}{*}{\textbf{TruthfulQA}}
 & ACC $\uparrow$ & $86.66_{0.58}$ & $86.66_{0.58}$ & $84.19_{0.72}$ & $85.63_{0.41}$ & $87.49_{0.39}$ & $\mathbf{94.60}_{0.32}$ \\
 & ECE $\downarrow$ & $18.19_{1.22}$ & $16.65_{1.31}$ & $9.93_{1.38}$  & $13.00_{1.35}$ & $11.35_{1.24}$ & $\mathbf{6.08}_{1.72}$ \\
\midrule
\multirow{2}{*}{\textbf{TriviaQA}}
 & ACC $\uparrow$ & $71.82_{0.63}$ & $71.82_{0.63}$ & $71.61_{0.78}$ & $72.05_{0.36}$ & $72.08_{0.36}$ & $\mathbf{72.40}_{0.93}$ \\
 & ECE $\downarrow$ & $8.56_{1.17}$  & $13.10_{1.38}$ & $19.23_{1.45}$ & $20.54_{1.28}$ & $11.01_{1.30}$ & $\mathbf{6.63}_{1.50}$ \\
\bottomrule
\end{tabular}
}
\end{table}

As shown in Table~\ref{table:generative}, UQ4CT achieves the lowest ECE on both benchmarks while maintaining the highest accuracy with only a single forward pass. The gains are particularly pronounced on TruthfulQA, where baselines show substantially higher ECE, indicating that our calibration loss improves MoE routing toward more accurate expert selection in generative settings.

\section{Discussion \& Conclusion}
In this work, we propose Functional-Level Uncertainty Quantification for Calibrated Fine-Tuning (UQ4CT), a framework for mitigating overconfidence in fine-tuned large language models. UQ4CT quantifies uncertainty from a functional perspective and incorporates it into a mixture-of-experts fine-tuning objective. The resulting calibration-aware loss substantially improves uncertainty calibration while preserving predictive accuracy. Across common-sense reasoning, domain-specific QA, and open-ended generation tasks, UQ4CT reduces Expected Calibration Error by more than 25\% without compromising accuracy under both in-distribution and out-of-distribution settings.

UQ4CT's calibration loss requires a correctness signal during training. For multiple-choice tasks, exact-match provides this naturally. Our exploratory experiments on open-ended generative QA (Sec.~\ref{sec:generative}) demonstrate that a simple token-level F1 threshold serves as a promising correctness proxy, suggesting that the functional-space perspective on calibration extends beyond classification-style tasks. However, this work primarily establishes the foundations of functional-level uncertainty quantification for calibrated fine-tuning; scaling to long-form reasoning or multi-step generation remains an important direction for future work. Potential avenues include leveraging reward models or learned evaluation functions as soft correctness proxies to broaden applicability.


\section*{Acknowledgement}
This work was supported in part by  U. S. Army Research Office
under Army-ECASE award W911NF-07-R-0003-03, the U.S. Department Of Energy, Office of Science, IARPA HAYSTAC Program, and NSF Grants \#2205093, \#2146343, \#2134274, CDC-RFA-FT-23-0069,  DARPA AIE FoundSci and DARPA YFA.

\newpage

\bibliography{ref}
\bibliographystyle{unsrtnat}

\newpage
\appendix
\onecolumn
\section{Appendix}
\subsection{Theoretical Derivation of the Method} \label{app:derivation}
In this section, we provide complete theoretical statements and proofs that are used in Sec.~\ref{sec:method} to derive our method.

\begin{fact}[Model Perturbation Structure, Restatement of Fact~\ref{fact:model}]
\label{fact:model_appnd}
Assume that in the residual connection architecture in each layer: $g_{k^\ell}^\ell (\mathbf{h}^\ell) = \mathbf{h}^\ell + E_{k^\ell}^\ell \big( f_{trans}^\ell (\mathbf{h}^\ell) \big)$, the Lipschitzness of the residual term $E_{k^\ell}^\ell \big( f_{trans}^\ell (\mathbf{h}^\ell) \big)$ is much smaller as compared to that of $\mathbf{h}^{\ell}$ itself: $\left\| E_{k^\ell}^\ell \big( f_{trans}^\ell (\hat{\mathbf{h}}^\ell) \big) - E_{k^\ell}^\ell \big( f_{trans}^\ell (\mathbf{h}^\ell) \big) \right\| = o\left( \| \hat{\mathbf{h}}^\ell - \mathbf{h}^\ell \| \right)$.
Under this regularity assumption, perturbations $\Delta f(x)$ to the model $f(x)$ approximately decomposes as:
\[
\Delta f(x) \approx \sum_{\ell=1}^{L} \sum_{k^\ell=1}^K \Delta \alpha_{k^\ell}^\ell(\mathbf{h}^\ell) \cdot g_{k^\ell}^\ell(\mathbf{h}^\ell).
\]
\end{fact}
\begin{proof}[Proof of Fact~\ref{fact:model} and~\ref{fact:model_appnd}]
In each layer, we can decompose the perturbation to the output as follows:
\begin{align*}
\Delta \mathbf{h}^{\ell+1} = \sum_{k^\ell=1}^K \left[ \Delta \alpha_{k^\ell}^\ell(\mathbf{h}^\ell) \cdot g_{k^\ell}^\ell(\mathbf{h}^\ell) + \alpha_{k^\ell}^\ell(\mathbf{h}^\ell) \cdot \Delta g_{k^\ell}^\ell(\mathbf{h}^\ell) \right], \quad \forall \ell=1,\dots, L,
\end{align*}
where the input $\mathbf{h}^1=x$ and the output $f(x)=\mathbf{h}^{L+1}$. 

Due to the residual connection architecture and our assumption on the regularity of the residual term, we have:
\[
\Delta g_{k^\ell}^\ell(\mathbf{h}^\ell) 
= \Delta \mathbf{h}^\ell + \Delta E_{k^\ell}^\ell \left( f_{trans}^\ell (\mathbf{h}^\ell) \right)
= \Delta \mathbf{h}^\ell + o\left(\Delta \mathbf{h}^\ell\right).
\]
Hence,
\begin{align*}
\Delta \mathbf{h}^{\ell+1} \approx \sum_{k^\ell=1}^K \left( \Delta \alpha_{k^\ell}^\ell(\mathbf{h}^\ell) \cdot g_{k^\ell}^\ell(\mathbf{h}^\ell) + \alpha_{k^\ell}^\ell(\mathbf{h}^\ell) \cdot \Delta \mathbf{h}^\ell \right)
= \sum_{k^\ell=1}^K \Delta \alpha_{k^\ell}^\ell(\mathbf{h}^\ell) \cdot g_{k^\ell}^\ell(\mathbf{h}^\ell) + \Delta \mathbf{h}^\ell.
\end{align*}
Expanding this recursion, the output perturbation can be approximated as:
\[
\Delta f(x) = \Delta \mathbf{h}^{L+1} \approx \sum_{\ell=1}^{L} \sum_{k^\ell=1}^K \Delta \alpha_{k^\ell}^\ell(\mathbf{h}^\ell) \cdot g_{k^\ell}^\ell(\mathbf{h}^\ell).
\]
\end{proof}

\begin{proposition}[Calibration Loss, Restatement of Proposition~\ref{prop:calib}]
\label{prop:calib_appnd}
Let the calibration risk be defined as the expectation of the calibration loss over the data distribution:
\begin{equation*}
\bar{\mathcal{L}}_{\mathrm{cal}} = \mathbb{E}_{(x,y^*)\sim\mathcal{D}}\left( \mathbbm{1}\{\mathrm{MixLoRA}(x) = y^*\} - \mathrm{FLC}(x) \right)^2.
\end{equation*}
If this calibration risk is optimized over the data distribution $\mathcal{D}$, then the optimal solution is $\mathrm{FLC}(x) = \mathbb{P}(\mathrm{MixLoRA}(x) = y^*(x))$; that is, the optimally trained FLC corresponds to the probability that the model's prediction is correct.
\end{proposition}

\begin{proof}[Proof of Proposition~\ref{prop:calib} and~\ref{prop:calib_appnd}]
Expanding the calibration risk, we have:
\begin{align*}
\bar{\mathcal{L}}_{\mathrm{cal}}
&= \mathbb{E}_{(x,y^*)\sim\mathcal{D}}\left(\mathbbm{1}\{\mathrm{MixLoRA}(x) = y^*\} - \mathrm{FLC}(x) \right)^2 \\
&= \mathbb{E}_{(x,y^*)\sim\mathcal{D}}
\left[ \mathbbm{1}\{\mathrm{MixLoRA}(x) = y^*\} \right]^2 \\
&\quad - 2\,\mathbb{E}_{x} \mathbb{E}_{y^*|x} \left[ \mathbbm{1}\{\mathrm{MixLoRA}(x) = y^*\} \cdot \mathrm{FLC}(x) \right] + \mathbb{E}_{x} \left[ \mathrm{FLC}(x)^2 \right] \\
&= C + \mathbb{E}_{x}\left[
\mathrm{FLC}(x)^2 - 2\,\mathbb{P}(\mathrm{MixLoRA}(x) = y^*(x)) \cdot \mathrm{FLC}(x)
\right],
\end{align*}
where $C$ is a constant independent of $\mathrm{FLC}(x)$. Thus, the calibration risk is minimized when $\mathrm{FLC}(x) = \mathbb{P}(\mathrm{MixLoRA}(x) = y^*(x))$.
\end{proof}

\subsection{Load Balancing Loss} \label{app:load}
We follow the load balancing loss in \citep{li2024mixlora}. Given $N$ experts indexed by $i=1$ to $N$ and a batch $B$ with $T$ tokens, the auxiliary loss is computed as: 
\begin{align}
    \mathcal{L}_{aux}=a \cdot N \cdot \sum^{N}_{i=1}\mathcal{F}_i \cdot \mathcal{P}_i,
\end{align}
where
\begin{align}
    \mathcal{F}_i = \frac{1}{T}\sum_{x\in B} \mathbbm{1}\{argmax_k \mathcal{R}(x)_k=i\}, \mathcal{P}_i = \frac{1}{T}\sum_{x\in B}\mathcal{R}(x)_i.
\end{align}

Here, $\mathcal{R}(\cdot)$ is the top-k router, $\mathcal{F}_i$ is the fraction of tokens dispatched to expert $i$ and $\mathcal{P}_i$
is the fraction of the router probability allocated for expert $i$. The final loss is multiplied by the expert count $N$ to keep the loss constant as the number of experts varies, and the constant term $a$ is set to $10^{-2}$as a multiplicative coefficient, which is large enough to ensure load balancing while remaining small enough not to overwhelm the primary objective.

\subsection{Experimental Results with Mistral-7B} \label{app:mistral}
In this section, we present the results using Mistral-7B \citep{jiang2023mistral}, a different decoder-based LLM backbone. Table \ref{table:mistral} shows the results of fine-tuning Mistral-7B on $3$ common-sense reasoning tasks and one domain-specific climate question-answering task.

For each of the tasks, UQ4CT effectively calibrates the parameter mixtures, leading to the best ECE performance in $3$ out of $4$ tasks. This indicates the robustness of UQ4CT across different LLMs.

\begin{table*}[t!]
\small
\centering
\caption{Performance comparison of different methods fine-tuned with Mistral-7B across $3$ common sense reasoning tasks and a domain-specific task. UQ4CT shows significant ECE improvements while maintaining high accuracy.}
\label{table:mistral}
\vspace{0.1cm}
\begin{tabular}{l||l|llll}
Metrics                                                                 & Methods  & ARC-E                 & ARC-C                  & OBQA                   & ClimateQA              \\ \hline \hline
\multirow{6}{*}{ACC $\uparrow$ \rule{0pt}{2.25ex}}   & LoRA     & $84.8_{0.47}$         & $70.2_{0.84}$          & $82.8_{0.62}$          & $72.5_{1.6}$           \\
                                                                        & MC Drop  & $84.6_{0.91}$         & $69.6_{0.76}$          & $82.6_{0.71}$          & $72.5_{1.6}$           \\
                                                                        & Ensemble & $84.2_{0.66}$         & $71.0_{1.41}$          & $82.5_{0.6}$           & $72.9_{2.88}$          \\
                                                                        & LA       & $82.4_{2.05}$         & $68.5_{3.31}$          & $82.5_{0.77}$          & $71.6_{1.56}$          \\
                                                                        & MixLoRA  & $85.5_{1.27}$         & $71.2_{1.75}$          & $83.3_{1.14}$          & $72.0_{1.69}$          \\
                                                                        & UQ4CT    & $\mathbf{85.9_{0.82}}$ & $\mathbf{74.4_{0.82}}$ & $\mathbf{83.7_{1.22}}$ & $\mathbf{73.2_{1.29}}$          \\ \hline \hline
\multirow{6}{*}{ECE $\downarrow$ \rule{0pt}{2.25ex}} & LoRA     & $9.46_{1.62}$         & $18.42_{1.91}$         & $13.3_{0.25}$          & $13.72_{2.62}$         \\
                                                                        & MC Drop  & $8.91_{1.35}$         & $18.38_{1.66}$         & $13.3_{0.31}$          & $13.72_{2.61}$         \\
                                                                        & Ensemble & $8.72_{1.49}$         & $17.0_{0.97}$          & $9.14_{2.82}$          & $12.86_{1.78}$         \\
                                                                        & LA       & $20.3_{5.7}$          & $21.27_{4.15}$         & $\mathbf{6.41_{3.22}}$ & $14.64_{2.21}$         \\
                                                                        & MixLoRA  & $8.16_{0.99}$         & $15.51_{3.86}$         & $10.53_{1.73}$         & $14.05_{3.09}$         \\
                                                                        & UQ4CT    & $\mathbf{5.7_{0.69}}$ & $\mathbf{7.04_{0.58}}$ & $\underline{7.92_{1.14}}$          & $\mathbf{11.4_{1.14}}$
\end{tabular}
\end{table*}

\subsection{Sensitivity Test on Calibration Term} \label{app:beta_sensitivity}
To understand the effectiveness of the calibration loss, we perform a sensitivity test of $\beta$ in Equation \ref{eq:total_loss}. This evaluates how our proposed calibration of parameter mixtures affects the overall model prediction and uncertainty quantification capabilities. We evaluate $\beta$ values ranging from $0$ to $1.5$, where $\beta=0$ resembles the original MixLoRA method.

\begin{table*}[t]
\centering
\caption{Performance of UQ4CT with varying $\beta$ values on the OBQA dataset. Prediction accuracy and uncertainty calibration improve with increasing $\beta$, highlighting the effectiveness of $\mathcal{L}_{\mathrm{cal}}$.}
\label{table:beta}
\vspace{0.1cm}
\resizebox{\linewidth}{!}{
\begin{tabular}{l c c c c c c c}
\toprule
\textbf{$\beta$} & $0$ & $0.2$ & $0.5$ & $0.8$ & $1$ & $1.2$ & $1.5$ \\
\midrule
\textbf{ACC} $\uparrow$ & $87.0_{2.85}$ & $87.1_{0.58}$ & $87.1_{0.29}$ & $87.3_{0.38}$ & $\mathbf{88.4_{0.35}}$ & $88.3_{0.57}$ & $87.5_{0.88}$ \\
\midrule
\textbf{ECE} $\downarrow$ & $12.7_{1.92}$ & $7.35_{0.75}$ & $7.69_{0.89}$ & $5.82_{1.12}$ & $\mathbf{3.34_{1.60}}$ & $6.31_{2.58}$ & $9.03_{0.82}$ \\
\bottomrule
\end{tabular}
}
\end{table*}

Results in Table \ref{table:beta} demonstrate the effectiveness of the calibration loss. Without calibration ($\beta=0$), the model exhibits high ECE. Even small values ($\beta=0.2$ or $0.5$) yield substantial ECE improvements. The best trade-off is at $\beta=1$, where the calibration term effectively optimizes the conditional parameter mixtures to fit the data distribution, achieving both the lowest ECE and highest accuracy. Beyond $\beta=1$, the calibration term begins to dominate the training objective, degrading both metrics.

\subsection{Number of Active Experts} \label{app:num_experts}
One important aspect of the LoRA MoE architecture is how many experts to activate. We investigate the performance impact of different numbers of active LoRA experts, evaluating the model with $1$ to $5$ active experts out of $8$ total.

\begin{table}[t]
\small
\centering
\caption{Performance comparison of UQ4CT with varying number of experts on OBQA dataset. Top-2 expert selection strategy grants best accuracy and calibration.}
\label{table:num_exp}
\vspace{0.1cm}
\begin{tabular}{l|l|l}
Top-K & ACC $\uparrow$         & ECE $\downarrow$      \\ \hline
Top-1 & $86.8_{0.59}$          & $7.54_{1.89}$         \\ \hline
Top-2 & $\mathbf{88.4_{0.35}}$ & $\mathbf{3.34_{1.60}}$ \\ \hline
Top-3 & $87.0_{0.77}$          & $5.68_{0.78}$          \\ \hline
Top-4 & $87.4_{0.51}$          & $7.52_{0.44}$         \\ \hline
Top-5 & $87.1_{0.48}$          & $6.16_{0.58}$         \\ \hline
\end{tabular}
\end{table}

As shown in Table \ref{table:num_exp}, $2$ active experts give the optimal performance in terms of accuracy and ECE scores. One expert alone cannot capture complicated functional relationships, while more than $2$ experts could potentially introduce redundant functional bases to the model, which deviates the output distribution from the data distribution, thus worsening predictive and calibration performance. Additionally, more active experts lead to a more flattened distribution across experts, which hardens the alignment of parameter mixtures during fine-tuning.

\subsection{Incremental Weighting on Calibration Term.} \label{app:incremental_weighting}
Due to the random initialization of LoRA experts, the predictions during early fine-tuning stage are likely to be incorrect as the model has little knowledge on the functional relationships regarding the data. Thus, it is intuitive to incrementally increase the weight parameter $\beta$ over the calibration term $\mathcal{L}_{\mathrm{cal}}$ in the training loss for the LoRA experts to learn before calibration. We conduct this study by incrementally increase $\beta$ from $0$ to $1$ within $50$ gradient steps during the early stage of fine-tuning:
\begin{equation}
    \beta = \min\left\{1, \frac{\mathrm{current\_grad\_step}}{50}\right\}.
\end{equation}
We choose $50$ gradient steps from our observation that training loss generally stabilizes after $50$ gradient steps, indicating the LoRA experts have learned some functional relationships from data.

As shown in Table \ref{table:ablation}, the incremental loss has significantly worse ECE performance across all tasks. This demonstrates the advantage of uncertainty calibration even in the early stage. In the beginning, the lack of functional relationships on the training data in LoRA experts lead to high uncertainty. Thus, UQ4CT encourages exploration over all LoRA experts while UQ4CT\_Incremental lacks it due to the small weighting in the beginning. 

\begin{table*}[t]
\small
\centering
\caption{Performance comparison of UQ4CT with and without incremental weighting. Incremental weighting has worse ECE performance while maintains similar accuracy.}
\label{table:ablation}
\vspace{0.1cm}
\begin{tabular}{l|| l |l l l l }
Metrics & Methods & ARC-E & ARC-C & OBQA & ClimateQA \\
\hline \hline
\multirow{1}{*}{ACC $\uparrow$} \rule{0pt}{2.25ex}
& UQ4CT & $88.66_{0.20}$ & $79.60_{1.21}$ & $88.40_{0.35}$ & $79.97_{0.85}$ \\
& UQ4CT\_Incremental & $87.15_{0.95}$ & $80.84_{1.03}$ & $88.53_{0.48}$ & $75.87_{2.89}$ \\
\hline\hline
\multirow{1}{*}{ECE $\downarrow$} \rule{0pt}{2.25ex}
& UQ4CT & $\mathbf{3.97_{0.78}}$ & $\mathbf{4.43_{0.82}}$ & $\mathbf{3.34_{1.60}}$ & $\mathbf{9.36_{2.77}}$ \\
& UQ4CT\_Incremental & $6.55_{1.42}$ & $10.02_{1.95}$ & $6.87_{1.68}$ & $14.16_{0.91}$ \\
\end{tabular}
\end{table*}




\subsection{Training Details} \label{app:implement_detail}
We train our model with total of $8$ LoRA experts, and select $2$ experts with the highest probability. For each expert, we use $rank=16$ and $alpha=32$. We use batch size of $16$ to train our model. For climate task, we set the learning rate to $5e-4$ and dropout rate to $0.1$ to incorporate the small dataset size. For other tasks, we use $2e-4$ as our learning rate with dropout $0.05$. We use AdamW as our optimizer and a cutoff length of $512$ for prompts during training. Our model is trained on A100 GPU, with 20GB GPU memory consumption per task. Training time is from $25$ to $50$ minutes depending on the task.

The experimental setup for single LoRA based models is similar with LoRA ranks set to $80$ to accommodate the MoE model size. For the ensemble baseline, we use an ensemble size of $8$ with $rank=16$. For Laplace-LoRA, we follow the Laplace hyperparameters in \href{https://github.com/MaximeRobeyns/bayesian_lora}{this Github Repository}.
\subsection{Expected Calibration Error} \label{app:ece}
Expected calibration error (ECE) is a commonly used metric to asses uncertainty quantification performance. ECE measures the alignment between prediction accuracy and model confidence through regrouping the predicted probabilities into $m$ bins. This method then computes the weighted average of the difference between average accuracy and confidence in each bin:
\begin{align}
    \text{ECE} = \sum_{m=1}^M \frac{|B_m|}{N} |\text{acc}(B_m) - \text{conf}(B_m)|,
\end{align}
where $|B_m|$ is the number of evaluated datapoints in bin $m$, acc and conf is calculated as following:
\begin{align}
    \text{acc}(B_m) = \frac{1}{|B_m|} \sum_{i\in B_m} \mathbf{1}(\hat{y}_i = y_i),
\end{align}
\begin{align}
    \text{conf}(B_m) = \frac{1}{|B_m|} \sum_{i\in B_m} P(\hat{y}_i).
\end{align}
In this paper, we use an ECE bin size of $15$, following the experiment setup in Laplace-LoRA \citep{yang2024bayesian}.

\subsection{MMLU Distribution Shift Dataset Composition} \label{app:OOD}
\begin{itemize}
    \item \textbf{Computer Science (CS)}:
    \begin{itemize}
        \item College Computer Science
        \item Computer Security
        \item High School Computer Science
        \item Machine Learning
    \end{itemize}
    \item \textbf{Engineering (Eng)}:
    \begin{itemize}
        \item Electrical Engineering
    \end{itemize}
    \item \textbf{Law}:
    \begin{itemize}
        \item International Law
        \item Jurisprudence
        \item Professional Law
    \end{itemize}
    \item \textbf{Health}:
    \begin{itemize}
        \item Anatomy
        \item Clinical Knowledge
        \item College Medicine
        \item Human Aging
        \item Nutrition
        \item Professional Medicine
        \item Virology
    \end{itemize}
\end{itemize}

\subsection{Prompt Perturbation Comparison} \label{app:ppc}
Here, we compare our method with SPUQ \citep{gao2024spuq}, which perturbs the prompt, aggregates predictions and confidences to measure uncertainty. We test SPUQ and UQ4CT with LLama3.1-8b as the base model. As shown in Table \ref{table:spuq_comparison}, SPUQ’s large ECE values suggest that simply aggregating predictions from perturbed prompts does not adequately calibrate model confidence, highlighting the limitations of prompt perturbation as an uncertainty quantification strategy for LLMs, especially for smaller models.

\begin{table*}[ht] \small \centering \caption{Performance comparison of UQ4CT and SPUQ across four tasks. UQ4CT achieves higher accuracy and substantially lower ECE than SPUQ.} \label{table:spuq_comparison} \vspace{0.1cm} \begin{tabular}{l|| l |l l l l } Metrics & Methods & ARC-E & ARC-C & OBQA & ClimateQA \\ \hline \hline \multirow{2}{*}{ACC $\uparrow$} \rule{0pt}{2.25ex} & SPUQ & $86.67$ & $73.94$ & $74.00$ & $67.80$ \\ & UQ4CT & $88.66_{0.20}$ & $79.60_{1.21}$ & $88.40_{0.35}$ & $79.97_{0.85}$ \\ \hline\hline \multirow{2}{*}{ECE $\downarrow$} \rule{0pt}{2.25ex} & SPUQ & $10.81$ & $7.40$ & $8.86$ & $14.68$ \\ & UQ4CT & $\mathbf{3.97_{0.78}}$ & $\mathbf{4.43_{0.82}}$ & $\mathbf{3.34_{1.60}}$ & $\mathbf{9.36_{2.77}}$ \\ \end{tabular} \end{table*}

\subsection{Peak Memory Analysis} \label{app:memory_analysis}
We provide a peak memory analysis using the same setup (Llama3.1-8B, OBQA, single A100, bf16 training). Following the BLoB paper's asymmetric design (Sec~3.1 of \citet{wang2024blob}), only $\mathbf{A}$ is Bayesianized with variance parameterized by $\mathbf{G} \in \mathbb{R}^{r \times n_\text{in}}$, and training uses $K=1$ sample with Flipout.

All three methods adapt 7 matrices per layer (Q, K, V, O, gate, up, down) across 32 layers. For Llama3.1-8B:
\begin{itemize}
    \item Q, K, V, O: $n_\text{in} = n_\text{out} = 4096$
    \item Gate, Up: $n_\text{in} = 4096$, $n_\text{out} = 14336$
    \item Down: $n_\text{in} = 14336$, $n_\text{out} = 4096$
\end{itemize}
Per layer: $\sum n_\text{in} = 4 \times 4096 + 2 \times 4096 + 14336 = 38912$, $\sum n_\text{out} = 4 \times 4096 + 2 \times 14336 + 4096 = 43008$, $\sum(n_\text{in} + n_\text{out}) = 81920$.

\paragraph{Trainable parameters (bf16, 2 bytes/param).}
\begin{table*}[h]
\small
\centering
\caption{Trainable parameter counts for LoRA, UQ4CT, and BLoB (bf16, 2 bytes/param).}
\label{table:trainable_params}
\vspace{0.1cm}
\resizebox{\linewidth}{!}{
\begin{tabular}{l l r r}
\toprule
\textbf{Method} & \textbf{Component} & \textbf{Params} & \textbf{bf16} \\
\midrule
\multirow{2}{*}{LoRA ($r{=}80$)}
& $\mathbf{B} \in \mathbb{R}^{n_\text{out} \times r}$: $80 \times 43008 \times 32 = 110.10$M & & \\
& $\mathbf{A} \in \mathbb{R}^{r \times n_\text{in}}$: $80 \times 38912 \times 32 = 99.61$M & 209.72M & 0.41 GB \\
\midrule
\multirow{2}{*}{UQ4CT ($K{=}8$, $r{=}16$)}
& 8 experts, each $\mathbf{B}_k, \mathbf{A}_k$: $8 \times 16 \times 81920 \times 32 = 335.54$M & & \\
& Routers $\mathbf{W}_r \in \mathbb{R}^{K \times d}$: $7 \times 8 \times 4096 \times 32 = 7.34$M & 342.88M & 0.64 GB \\
\midrule
\multirow{3}{*}{BLoB ($r{=}80$)}
& $\mathbf{B} \in \mathbb{R}^{n_\text{out} \times r}$ (deterministic): $80 \times 43008 \times 32 = 110.10$M & & \\
& $\mathbf{M} \in \mathbb{R}^{r \times n_\text{in}}$ (mean of $\mathbf{A}$): $80 \times 38912 \times 32 = 99.61$M & & \\
& $\mathbf{G} \in \mathbb{R}^{r \times n_\text{in}}$ (std param of $\mathbf{A}$, $\Omega = \mathbf{G}^2$): $80 \times 38912 \times 32 = 99.61$M & 309.33M & 0.58 GB \\
\bottomrule
\end{tabular}
}
\end{table*}

Note: $\sum n_\text{out} > \sum n_\text{in}$ because the FFN gate/up layers project from 4096 to 14336 (large $n_\text{out}$), while the down layer reverses this. The asymmetry does not cancel, so $\mathbf{B}$ (which depends on $n_\text{out}$) has more parameters than $\mathbf{M}$ or $\mathbf{G}$ (which depend on $n_\text{in}$).

\paragraph{Peak GPU memory breakdown.}
\begin{table*}[h]
\small
\centering
\caption{Peak GPU memory breakdown for LoRA, UQ4CT, and BLoB on Llama3.1-8B (OBQA, single A100, bf16).}
\label{table:peak_memory}
\vspace{0.1cm}
\begin{tabular}{l r r r}
\toprule
\textbf{Component} & \textbf{LoRA} & \textbf{UQ4CT} & \textbf{BLoB} \\
\midrule
Base model (frozen, bf16) & 16.00 GB & 16.00 GB & 16.00 GB \\
Adapter params (bf16, 2B/param) & 0.41 GB & 0.64 GB & 0.58 GB \\
Optimizer states (fp32, 8B/param) & 1.56 GB & 2.57 GB & 2.31 GB \\
Gradients (bf16, 2B/param) & 0.39 GB & 0.64 GB & 0.58 GB \\
Sampling buffers ($\mathbf{E}$, sampled $\mathbf{A}$, each $r \times n_\text{in}$) & --- & --- & 0.38 GB \\
Activations/KV cache (shared) & ${\sim}$1.50 GB & ${\sim}$1.50 GB & ${\sim}$1.50 GB \\
\midrule
\textbf{Estimated peak} & $\mathbf{{\sim}19.86}$ \textbf{GB} & $\mathbf{{\sim}21.35}$ \textbf{GB} & $\mathbf{{\sim}21.35}$ \textbf{GB} \\
\bottomrule
\end{tabular}
\end{table*}

As the analysis above shows, UQ4CT and BLoB have nearly identical peak memory (${\sim}$21.35~GB), both ${\sim}$1.49~GB (${\sim}$7.50\%) above standard LoRA, which is minimal and manageable. UQ4CT's overhead comes from additional expert parameters, while BLoB's comes from variance parameters ($\mathbf{G}$) plus sampling buffers.

\subsection{Open-Ended Generative Setup} \label{app:generative_setup}
To extend UQ4CT beyond multiple-choice settings, we adapt the calibration loss (Eq.~\ref{eq:calibration}) for open-ended generation by replacing the exact-match indicator with an F1-score-based correctness proxy:
\begin{equation} \label{eq:calibration_gen}
    \mathcal{L}_{\mathrm{cal}} = \left(\mathbbm{1}\{F_1(\mathrm{MixLoRA}(x), y^*) > \tau\} - \mathrm{FLC}(x) \right)^2,
\end{equation}
where $\tau = 0.3$. If the token-level F1 between the generated answer and the ground truth exceeds $\tau$, we count the prediction as correct. The threshold $\tau = 0.3$ is motivated by \citet{alur2024human}, who show that this level aligns with human judgment boundaries for partial-credit correctness in open-ended QA.

We evaluate on TruthfulQA \citep{lin2021truthfulqa} and TriviaQA \citep{joshi2017triviaqa} using Llama-3.1-8B-Instruct. To ensure a controlled comparison, we first fine-tune the base model with MixLoRA (rank=16, 8 experts), then apply all UQ baseline methods to this shared fine-tuned backbone. The baselines include \textbf{Sequence Probability (SeqProb)} and \textbf{Verbalized Confidence} \citep{xiong2024can}, \textbf{Verbalized Consistency} \citep{xiong2024can}, \textbf{Semantic Entropy} \citep{farquhar2024detecting}, and \textbf{Semantic Density} \citep{qiu2024semantic}. Sampling-based methods (marked with *) use $N=5$ samples per question. All baselines use average generated token probability as confidence.

\end{document}